\documentclass[a4paper,reqno]{amsart}
\usepackage{cite}
\usepackage{amsmath,amssymb,amsfonts}
\usepackage{algorithmic}
\usepackage{graphicx}
\usepackage{textcomp}
\usepackage{xcolor}
\usepackage{microtype}
\usepackage{tikz}
\usetikzlibrary{calc, arrows}

\usepackage{amsthm}
\usepackage[foot]{amsaddr}

\newcommand{\Real}{\mathbb{R}}

\newcommand{\triplebar}{|\kern-0.1em|\kern-0.1em|}

\newcommand{\ballvolume}[1][d]{V_{#1}^{\operatorname{ball}}}
\newcommand{\ballarea}[1][d]{S_{#1}^{\operatorname{ball}}}
\newcommand{\capvolume}[1][d]{V_{#1}^{\operatorname{cap}}}
\newcommand{\capvolumefraction}[1][d]{W_{#1}^{\operatorname{cap}}}

\newcommand{\capareafraction}[1][d]{T_{#1}^{\operatorname{cap}}}
\newcommand{\hypgeometric}{{}_2F_1}

\newcommand{\given}{\,|\,}
\newcommand{\suchthat}{:}

\newtheorem{theorem}{Theorem}

\newtheorem{remark}[theorem]{Remark}

\newtheorem{definition}[theorem]{Definition}
\newtheorem{corollary}[theorem]{Corollary}

\begin{document}

\title[Relative intrinsic dimensionality is intrinsic to learning]{Relative intrinsic dimensionality\\is intrinsic to learning}

\author[O. J. Sutton]{
    Oliver J. Sutton\textsuperscript{\rm 1}
}
\author[Q. Zhou]{
    Qinghua Zhou\textsuperscript{\rm 1}
}
\author[A. N. Gorban]{
    Alexander N. Gorban\textsuperscript{\rm 2}
}
\author[I. Y. Tyukin]{
    Ivan Y. Tyukin\textsuperscript{\rm 1}
}
\address{
    \textsuperscript{\rm 1}Department of Mathematics,
    King's College London,
    WC2R 2LS
}
\address{
    \textsuperscript{\rm 2}School of Computing and Mathematical Sciences, 
    University of Leicester,
    LE1 7RH
}
\email{oliver.sutton@kcl.ac.uk}
\email{qinghua.zhou@kcl.ac.uk}
\email{ivan.tyukin@kcl.ac.uk}
\email{a.n.gorban@leicester.ac.uk}

\maketitle

\begin{abstract}
    High dimensional data can have a surprising property: pairs of data points may be easily separated from each other, or even from arbitrary subsets, with high probability using just simple linear classifiers.
    However, this is more of a rule of thumb than a reliable property as high dimensionality alone is neither necessary nor sufficient for successful learning.
    Here, we introduce a new notion of the \emph{intrinsic dimension} of a data distribution, which precisely captures the separability properties of the data.
    For this intrinsic dimension, the rule of thumb above becomes a law: high intrinsic dimension guarantees highly separable data.
    We extend this notion to that of the \emph{relative} intrinsic dimension of two data distributions, which we show provides both upper and lower bounds on the probability of successfully learning and generalising in a binary classification problem.

    \noindent\textbf{Keywords}: intrinsic dimensionality, classification problems, high dimensional data.
\end{abstract}

\section{Introduction}
\looseness=-1
A \emph{blessing of dimensionality} often ascribed to data sampled from genuinely high dimensional probability distributions is that pairs (and even arbitrary compact subsets) of points may be easily separated from one another with high probability~\cite{gorban2017stochastic,kainen2020quasiorthogonal, Gorban:2018, tyukin:kernelStochasticSeparation, anderson2014more, gorban2016blessing,donoho2009observed}.
Such a property is naturally highly appealing for Machine Learning and Artificial Intelligence, since it suggests that if sufficiently many attributes can be obtained for each data point, then classification is a significantly easier task.

\begin{figure}
    \begin{center}
        \begin{tikzpicture}[scale=0.75]
            \def\outerRadius{2cm}
            \def\innerRadius{2cm}
            \def\separatorX{0.85}
            \coordinate (y) at (3, 0);
            \coordinate (c) at (0, 0);
            \coordinate (separatorTop) at (0.75, 2.2);
            \coordinate (separatorBottom) at (0.75, -2.2);
            \coordinate (midpoint) at ($(c)!0.5!(y)$);
            \coordinate (smallCircleEdge) at ($(midpoint) - (0,\innerRadius)$);
            \coordinate (bigCircleEdge) at ($(c) + (0,\outerRadius)$);
            \coordinate (smallRadiusMidpoint) at ($(midpoint)!0.5!(smallCircleEdge)$);
            \coordinate (bigRadiusMidpoint) at ($(c)!0.5!(bigCircleEdge)$);
            \coordinate (separatorMidpoint) at ($(\separatorX*\outerRadius, 0)$);
            \def\firstcircle{(0,0) circle (\outerRadius)}
            \def\secondcircle{(midpoint) circle (\innerRadius)}
            \def\firstrectangle{(-\outerRadius,-\outerRadius) rectangle ($(\separatorX*\outerRadius,\outerRadius)$)}
            \def\secondrectangle{($(\separatorX*\outerRadius,-\outerRadius)$) rectangle ($(\outerRadius,\outerRadius)$)}

            \colorlet{circle edge}{black}
            \colorlet{circle area}{yellow!20}
            \colorlet{rectangle edge}{blue!20}
            
            \tikzset{filled/.style={fill=circle area},
            outline/.style={draw=circle edge, thick}}
            \draw[outline] \firstcircle;
            \draw[outline] \secondcircle;
            \fill (c) circle [fill, radius=2pt, anchor=west] node[anchor=east] {$c_1$};
            \fill (midpoint) circle [fill, radius=2pt, anchor=south] node[anchor=south] {$c_2$};
            \draw (smallRadiusMidpoint) node[anchor=east] {$1$};
            \draw (bigRadiusMidpoint) node[anchor=north west] {$1$};
            \draw[<->, dashed] ($(midpoint) - (0, 0.1)$) -- (smallCircleEdge);
            \draw[<->, dashed] ($(c) + (0, 0.1)$) -- (bigCircleEdge);
            \draw[<->, dashed] ($(c) + (0.1, 0)$) -- ($(midpoint) - (0.1, 0)$);
            \draw[dotted] (separatorTop) -- (separatorBottom);
            \draw ($(midpoint)!0.5!(c)$) node[anchor=north west] {$\epsilon$};
        \end{tikzpicture}
    \end{center}
    \caption{Two unit balls separated by distance epsilon, and the optimal classifier (dotted) separating the two.}
    \label{fig:closeSpheres}
\end{figure}
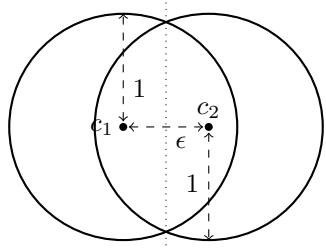

However, although this provides a useful rule of thumb, it is far from a complete description of the behaviour which may be expected of high dimensional data, and a simple experiment shows that the precise relationship between data dimension and classification performance is more subtle (see also~\cite{gorban2018correction}, Theorem 5 and Corollary 2).
Suppose that data are sampled from two classes, each described by a uniform distribution in a unit ball in $\mathbb{R}^d$, and that the centres of these balls are at distance $\epsilon \geq 0$ from one another, as shown in Figure~\ref{fig:closeSpheres}.
The classifier which offers the optimal (balanced) accuracy in this case is given by the hyperplane which is normal to the vector connecting the two centres and positioned half way between them.
In Figure~\ref{fig:ballSeparationProbability} we plot the accuracy of this classifier as a function of the distance separating the two centres for data sampled from various different ambient dimensions $d$.
The insight behind the blessing of dimensionality described above is immediately clear: when the data is sampled in high dimensions, for values of $\epsilon$ greater than some threshold value $\epsilon_0(d)$ depending on the ambient dimension $d$, the accuracy of this simple linear classifier is virtually 100\%.
Yet, what this simplified viewpoint misses is that, for $\epsilon < \epsilon_0(d)$ the probability of correctly classifying a given point sharply drops to close to 50\%, demonstrating that raw dimensionality alone is no panacea for data classification\footnote{Moreover, standard dimensionality reduction techniques, such as Principle Components Analysis, would not have any effect here since the data are uniformly sampled from $d$-dimensional balls.}.
On the other hand, data sampled even in 1 dimension may be accurately classified when the centre separation $\epsilon$ is sufficiently large: for $\epsilon \geq 2$ (when the two unit balls are disjoint), the two data sets are fully separable in any dimension.

What this simple thought experiment demonstrates is a fact which is not taken into account by previous work, such as~\cite{fewShotNonlinear}: 

{\it Determining whether data distributions are separable from each other must depend on a \emph{relative} property of the two, and even genuine high dimensionality\footnote{In the sense that dimensionality reduction techniques cannot be applied to find an equivalent lower dimensional representation of the data.} alone is neither a necessary nor sufficient condition for data separability}

\begin{figure}
	\centering
    \includegraphics[width=0.6\linewidth]{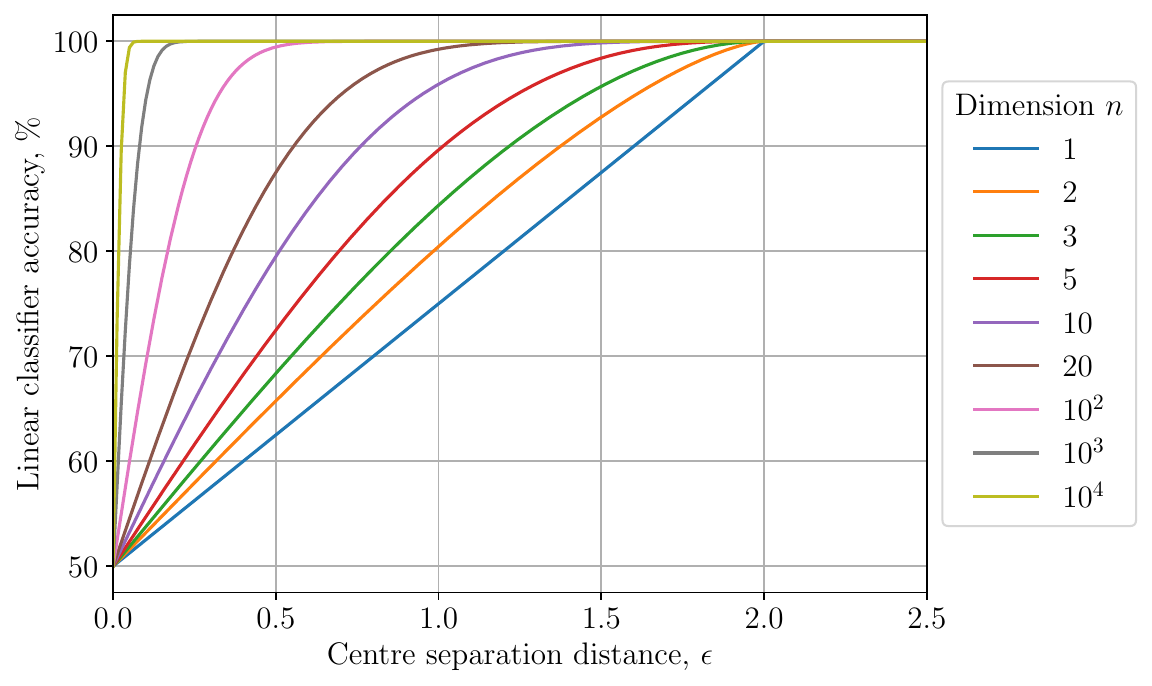}
    \caption{\looseness=-1 Accuracy of the best linear classifier separating data uniformly sampled from two balls with unit radius and centres in $\Real^n$ separated by distance $\epsilon$ for different dimensions $n$.}
    \label{fig:ballSeparationProbability}
\end{figure}

To lay the foundations of our approach, we propose the new concept of the \emph{intrinsic dimension} of a data distribution, based directly on the separability properties of sampled data points.

\begin{definition}[Intrinsic dimension]\label{def:intrinsicDim}
    \looseness=-1
    We say that data sampled from a distribution $\mathcal{D}$ on $\mathbb{R}^d$ has \emph{intrinsic dimension} $n(\mathcal{D}) \in \Real$ with respect to a centre $c \in \Real^d$ if
    \begin{equation}\label{eq:sepdimdef}
        P(x, y \sim \mathcal{D} : (x - y, y - c) \geq 0) = \frac{1}{2^{n(\mathcal{D}) + 1}}.
    \end{equation}
\end{definition}

This definition is designed in such a way that the rule of thumb in the blessing of dimensionality described above becomes a \emph{law of high intrinsic dimension}: points sampled from a distribution with high intrinsic dimension are highly separable.
The definition is calibrated so that the uniform distribution $\mathcal{U}(\mathbb{B}_d)$ on a $d$-dimensional unit ball $\mathbb{B}_d$ satisfies $n(\mathcal{U}(\mathbb{B}_d)) = d$ (see Theorem~\ref{thm:separability}), although alternative normalisations are possible, and by symmetry $n(\mathcal{D}) \geq 0$ for all distributions $\mathcal{D}$. For $c=0$, the expression $(x - y, y - c) \geq 0$ in the left-hand side of~\eqref{eq:sepdimdef} is simply a statement that $x$ and $y$ are Fisher-separable \cite{gorban2018correction}.

Based on the same principle, we further introduce the concept of the \emph{relative intrinsic dimension} of two data distributions, which directly describes the ease of separating data distributions.

\begin{definition}[Relative intrinsic dimension]
    We say that data sampled from a distribution $\mathcal{D}$ on $\mathbb{R}^d$ has \emph{relative intrinsic dimension} $n(\mathcal{D}, \mathcal{D}^{\prime}) \in \Real$ to data sampled from a distribution $\mathcal{D}^{\prime}$ on $\mathbb{R}^d$, with respect to a centre $c \in \Real^d$, if
    \begin{equation}\label{eq:sepdimdef:relative}
        P(x \sim \mathcal{D}^{\prime}, y \sim \mathcal{D} : (x - y, y - c) \geq 0) = \frac{1}{2^{n(\mathcal{D}, \mathcal{D}^{\prime}) + 1}}.
    \end{equation}
\end{definition}

The relative intrinsic dimension is not symmetric, and satisfies $n(\mathcal{D}, \mathcal{D}^{\prime}) \geq -1$, with negative values indicating that $\mathcal{D}$ has lower intrinsic dimension than $\mathcal{D}^{\prime}$, and data distributions with a low relative intrinsic dimension may be separated from distributions with a high relative intrinsic dimension.

To illustrate this, consider our previous experiment as an example and let $X = \mathcal{U}(B_1)$ and $Y = \mathcal{U}(B_2)$, where $B_1=\mathbb{B}_d(1,c_1) \subset \mathbb{R}^d$ and $B_2=\mathbb{B}_d(1,c_2) \subset \mathbb{R}^d$ are the unit balls centered at $c_1$ and $c_2$ respectively, and pick the centre $c = c_1$.
When $\epsilon = \|c_1 - c_2\| \geq 2$ (the case when the data distributions are completely separable), we have $n(Y, X) = \infty$.
This implies that points $y$ sampled from $Y$ can be separated from points sampled from points $x$ sampled from $X$ with certainty.
The relative intrinsic dimension $n(X, Y)$ is an increasing function of the dimension of the ambient space in which the data is sampled with $n(X, Y) = 0$ in 1 dimension, implying that it becomes easier to separate points in $X$ from points in $Y$ as the dimension increases.
These values of the relative intrinsic dimensions suggest that points from $Y$ can easily be separated from points in $X$ by hyperplanes normal to $y - c_1$, while hyperplanes normal to $x - c_1$ do not separate $X$ from $Y$.

Although the asymmetry may be slightly surprising at first, it is simply reflecting the asymmetric choice of centre $c = c_1$, which is located at the heart of the $X$ distribution.
The relative intrinsic dimensions described above would be reversed for $c = c_2$ and would be equal for $c = \frac{1}{2}(c_1 + c_2)$.
A justification for this definition of relative intrinsic dimension is given by Theorem~\ref{thm:intrinsicDimensionAndLearning}, where it is shown (in a slightly generalised setting) that these concepts of intrinsic dimension provide upper and lower bounds on classifier accuracy, indicating that it is indeed necessary and sufficient for learning.

There is a rich history of alternative charaterisations of the dimension of a data set, with each contribution typically aimed to solve a particular problem.
For example, conventional Principle Components Analysis aims to detect the number of independent attributes which are actually required to represent the data, leading to compressed representations of the same data.
However, as discussed above, the representational dimension of a data set does not necessarily give an indication of how easy it is to learn from.
Several other notions of dimensionality are captured in the \texttt{scikit-dimension} library~\cite{scikitdimension}.
Perhaps the most similar notion of dimension to that which we propose here is the Fisher Separability Dimension~\cite{Albergante:2019}, which is also based on the separability properties of data yet first requires a whitening step to normalise the data covariance to an identity matrix.
This whitening step has both advantages and disadvantages: although it brings invariance to the choice and scaling of the basis, it disrupts the intrinsic geometry of the data.
The Fisher Separability Dimension also does not address the important question of the \emph{relative} dimension of data distributions and samples, which we argue is a concept fundamental to learning.

\looseness=-1
Our approach may appear reminiscent of Kernel Embeddings, through which nonlinear kernels are used to embed families of data distributions into a Hilbert space structure~\cite{Smola:2007}. Although Kernel Embeddings and our work are motivated by very different classes problems, the common fundamental
focus is on understanding the properties of a data distribution through the evaluation of (nonlinear) functionals of the distribution. Here we demonstrate how a single, targeted, property appears to encode important information about the separability properties of data.

An interesting question which arises from this work is how well the (relative) intrinsic dimension can be estimated from data samples directly.
If it can be, then this could provide a new tool for selecting appropriate feature mappings for data and shine a new light on the training of neural networks.
We briefly investigate this in Section~\ref{sec:polynomialKernels}, where we show that high order polynomial feature maps can actually be detrimental to the separability of data.

\section{Separability of uniformly distributed data}
We investigate the separability properties of data sampled from a uniform distribution in the unit ball in various dimensions.
This provides the basis for our definition of intrinsic dimension.

To simplify the presentation of our results, we introduce the following geometric quantities related to spheres in high dimensions.
The volume of a ball with radius $r$ in $d$ dimensions is denoted by
$$
    \ballvolume[d](r) = \frac{\pi^{d/2}r^d}{\Gamma(\frac{d}{2} + 1)},
$$
and the surface area of the same ball is denoted by
$$
    \ballarea[d](r) = \frac{d\pi^{d/2}r^{d-1}}{\Gamma(\frac{d}{2} + 1)}.
$$
Similarly, the volume of the spherical cap with height $h$ of the same sphere (i.e. the set of points $\{x \in \Real^d \suchthat \|x\| \leq r \text{ and } x_0 \geq r - h\}$) is given by
$
    \capvolume[d](r, h) = \ballvolume[d](r) W_d^{\operatorname{cap}}(r, h),
$
where
$$
    W_d^{\operatorname{cap}}(r, h) = 
    \begin{cases}
        0 
        & \text{for } h \leq 0,
        \\
        \frac{1}{2} I_{(2rh - h^2) / r^2}(\frac{d + 1}{2}, \frac{1}{2}) 
        & \text{for } 0 < h \leq r,
        \\
        1 - W_d^{\operatorname{cap}}(r, 2r - h) 
        & \text{for } r < h \leq 2r,
        \\
        1 
        & \text{for } 2r < h,
    \end{cases}
$$
represents the fraction of the volume of the unit ball contained in the spherical cap.
The function $I_x(a, b) = B(a, b)^{-1} \int_0^x t^{a-1}(1-t)^{b-1} dt$ denotes the regularised incomplete beta function, where $B(a, b) = B(1; a, b) = \frac{\Gamma(a)\Gamma(b)}{\Gamma(a + b)}$ is the standard beta function.

\begin{figure}
	\centering
    \includegraphics[width=0.6\linewidth]{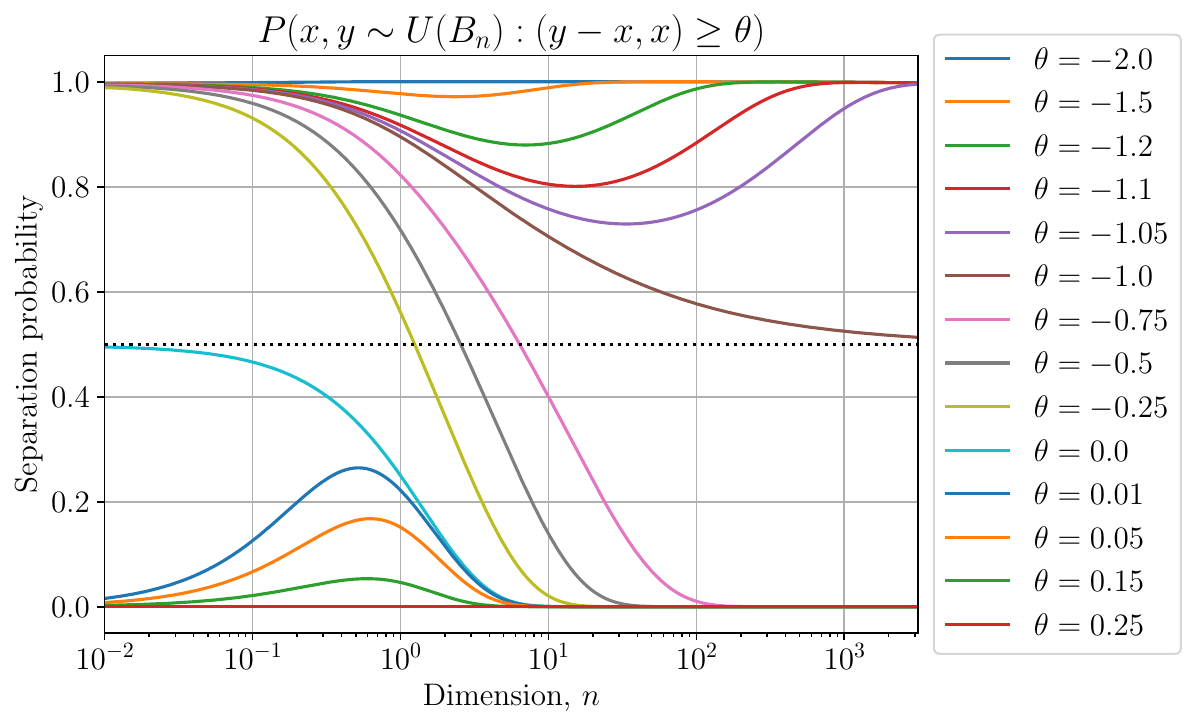}
    \caption{The behaviour of $f_{\theta}(d)$, formally extended to non-integer values of $d$, for various values of $\theta$. The function is only invertible for $-1 \leq \theta \leq 0$, and we note the asymptote of $\frac{1}{2}$ as $d \to 0$ when $\theta = 0$ and as $d \to \infty$ when $\theta = -1$.}
    \label{fig:fTheta}
\end{figure}

\begin{theorem}[Separability of uniformly sampled points]\label{thm:separability}
    Let $\theta \in \Real$, let  $d$ be a positive integer and suppose that $x, y \sim \mathcal{U}(\mathbb{B}_d(1, c))$, define
    \begin{equation}\label{eq:lengthDefs}
        R_{\theta}(t) = \max \Big\{ \frac{t^2}{4} - \theta, 0 \Big\}^{\frac{1}{2}},
        \,
        a_{\theta}(t) = \frac{1 - R_{\theta}^2(t)}{t} - \frac{t}{4}, 
    \end{equation}
    and
    \begin{equation}
        b_{\theta}(t) = 1 - a_{\theta}(t) - \frac{t}{2},
    \end{equation}
    and let 
    \begin{align}\label{eq:dimFunctionDef}
        f_{\theta}(d) 
        &
        = \int_{0}^1 dt^{d-1} \big(\capvolumefraction[d](1, b_{\theta}(t))
        + R_{\theta}^d(t) \capvolumefraction[d](R_{\theta}(t), R_{\theta}(t) + a_{\theta}(t)) \big)  dt.
    \end{align}
    Then
    \begin{equation}\label{eq:thresholdSepProb}
        P(x, y : (y - x, x - c) \geq \theta) = f_{\theta}(d),
    \end{equation}
    and, in particular,
    \begin{equation}\label{eq:zeroSepProbability}
        P(x, y : (y - x, x - c) \geq 0) = \frac{1}{2^{d+1}}.
    \end{equation}
    Furthermore, $f_{\theta}$ may be simplified in the following cases as
    \begin{equation}
        f_{\theta}(d) = 
        \begin{cases}
            1 &\text{ for } \theta \leq -2,
            \\
            \frac{1}{2^{d + 1}}
            &\text{ for } \theta = 0,
            \\
            \int_{2\theta^{1/2}}^1 dt^{d-1} \Big( \frac{t^2}{4} - \theta \Big)^{d/2} dt
            &\text{ for } 0 < \theta < \frac{1}{4},
            \\
            0 &\text{ for } \frac{1}{4}\leq \theta.
        \end{cases}
    \end{equation}
    and $f_{\theta}(d) \geq \frac{1}{2}$ for $\theta \leq -1$.
\end{theorem}
\begin{figure}
    \begin{center}
        \begin{tikzpicture}[scale=0.75]
            \def\outerRadius{3cm}
            \def\innerRadius{2.2cm}
            \def\separatorX{0.85}
            \coordinate (y) at (2, 0);
            \coordinate (c) at (0, 0);
            \coordinate (p) at (\separatorX*\outerRadius, 1.57);
            \coordinate (q) at (\separatorX*\outerRadius, 0);
            \coordinate (midpoint) at ($(c)!0.5!(y)$);
            \coordinate (smallCircleEdge) at ($(midpoint) - (0,\innerRadius)$);
            \coordinate (bigCircleEdge) at ($(c) + (0,\outerRadius)$);
            \coordinate (smallRadiusMidpoint) at ($(midpoint)!0.5!(smallCircleEdge)$);
            \coordinate (bigRadiusMidpoint) at ($(c)!0.5!(bigCircleEdge)$);
            \coordinate (separatorMidpoint) at ($(\separatorX*\outerRadius, 0)$);
            \def\firstcircle{(0,0) circle (\outerRadius)}
            \def\secondcircle{(midpoint) circle (\innerRadius)}
            \def\firstrectangle{(-\outerRadius,-\outerRadius) rectangle ($(\separatorX*\outerRadius,\outerRadius)$)}
            \def\secondrectangle{($(\separatorX*\outerRadius,-\outerRadius)$) rectangle ($(\outerRadius,\outerRadius)$)}

            \colorlet{circle edge}{black}
            \colorlet{circle area}{yellow!20}
            \colorlet{rectangle edge}{blue!20}
            
            \tikzset{filled/.style={fill=circle area},
            outline/.style={draw=circle edge, thick}}

            \begin{scope}
                \clip \firstcircle;
                \begin{scope}
                    \clip \firstrectangle;
                    \fill[filled] \secondcircle;
                    \draw[outline] \secondrectangle;
                \end{scope}
                \begin{scope}
                    \colorlet{circle area}{blue!20}
                    \clip \secondrectangle;
                    \fill[filled] \firstcircle;
                \end{scope}
            \end{scope}
            \draw[outline] \firstcircle;
            \draw[outline] \secondcircle;
            \fill (y) circle [fill, radius=2pt, anchor=north west] node[anchor=south] {$y$};
            \fill (c) circle [fill, radius=2pt, anchor=west] node[anchor=east] {$O$};
            \fill (p) circle [fill, radius=2pt] node[anchor=south west] {$p$};
            \fill (q) circle [fill, radius=2pt] node[anchor=north] {$q$};
            \fill (midpoint) circle [fill, radius=2pt, anchor=south] node[anchor=south] {$\frac{y}{2}$};
            \draw (smallRadiusMidpoint) node[anchor=east] {$R$};
            \draw (bigRadiusMidpoint) node[anchor=north west] {$1$};
            \draw[<->, dashed] ($(midpoint) - (0, 0.1)$) -- (smallCircleEdge);
            \draw[<->, dashed] ($(c) + (0, 0.1)$) -- (bigCircleEdge);
            \draw[<->, dashed] ($(midpoint) + (0.1, 0)$) -- ($(q) - (0.1, 0)$);
            \draw ($(midpoint)!0.5!(separatorMidpoint)$) node[anchor=south east] {$a$};
            \draw[<->, dashed] ($(q) + (0.1, 0)$) -- (\outerRadius, 0);
            \draw ($(\outerRadius, 0)!0.5!(separatorMidpoint)$) node[anchor=south] {$b$};
        \end{tikzpicture}
    \end{center}
    \caption{The shaded area is the volume computed in the proof of Theorem~\ref{thm:separability}. The two different shading colours indicate the two spherical caps used in the proof.}
    \label{fig:sphericalCaps}
\end{figure}
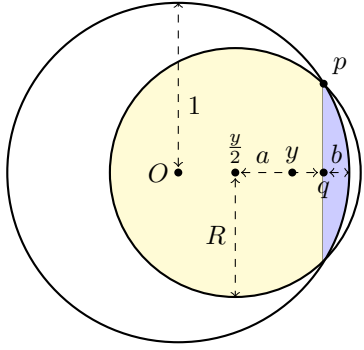
\begin{proof}
    Without loss of generality, we suppose that $c = 0$, and consider points $x, y \sim \mathcal{U}(\mathbb{B}_d$).
    Rearranging terms, we observe that
    $$
        (y - x, x) = \frac{1}{4}\|y\|^2 - \|x - \frac{y}{2}\|^2,
    $$
    and therefore, for fixed $y$, the set of $x$ satisfying $(y - x, x - c) \geq \theta$ may be similarly described as those points $x$ contained within the ball
    $$
        \|x - \frac{y}{2}\|^2 \leq R(\|y\|) = \max \Big\{ \frac{1}{4}\|y\|^2 - \theta, 0 \Big\}.
    $$
    Combining this with the condition that $x \sim \mathbb{B}_d(1, 0)$, we find that $x$ belongs to the intersection of the balls
    \begin{equation}\label{eq:ballIntersection}
        \{ x \in \Real^d \suchthat \|x\| \leq 1 \} \cap \Big\{ x \in \Real^d \suchthat \|x - \frac{y}{2}\|^2 \leq R_{\theta}(\|y\|) \Big\}.
    \end{equation}
    This may be expressed as the union of two spherical caps, as depicted in Figure~\ref{fig:sphericalCaps}.
    Comparing the triangles $O, p, q$ and $\frac{y}{2}, p, q$ shows that the lengths $a$ and $b$ in the Figure are exactly those defined in~\eqref{eq:lengthDefs} with $t = \|y\|$.
    Since $y$ only appears through its norm, we deduce that
    \begin{align*}
        P(x \suchthat \,&(y - x, x) \geq \theta \given \|y\|)
        =
        P(x \suchthat (y - x, x) \geq \theta \given y) 
        \\&
        = 
        \frac{\capvolume(R_{\theta}(\|y\|), R_{\theta}(\|y\|) + a_{\theta}(\|y\|)) + \capvolume(1, b_{\theta}(\|y\|))}{\ballvolume(1)},
    \end{align*}
    The result~\eqref{eq:thresholdSepProb} follows by applying the law of total probability, which implies
    \begin{align*}
        &P(x, y \suchthat (y - x, x) \geq \theta) 
        = \int_{0}^1 P(x : (y - x, x) \geq \theta \given \|y\| = t) p_{\|y\|}(t) dt,
    \end{align*}
    where $p_{\|y\|}(t) = \frac{\ballarea(t)}{\ballvolume(1)}$ is the density associated with $\|y\|$ for $y \sim \mathcal{U}(\mathbb{B}_d)$.

    When $\theta \geq 0$, the ball centered at $\frac{y}{2}$ is entirely contained within $\mathbb{B}_d$, and so
    \begin{align*}
        P(x, y \suchthat (y - x, x) \geq \theta) 
        &= 
        \int_{0}^1 \frac{\ballarea(t) \ballvolume(R_{\theta}(t))}{(\ballvolume(1))^2} dt
       \\&
        =
        \int_0^1 dt^{d-1} \max\Big\{ \frac{t^2}{4} - \theta, 0 \Big\}^{d/2} dt.
    \end{align*}
    Since the integrand is zero for $t \leq 2\theta^{1/2}$, for $\theta \in (0, \frac{1}{4})$ we have
    $$
        P(x, y \suchthat (y - x, x) \geq \theta) 
        =
        \int_{2\theta^{1/2}}^1 dt^{d-1} \Big( \frac{t^2}{4} - \theta \Big)^{d/2} dt.
    $$
    Moreover, $P(x, y \suchthat (y - x, x) \geq \theta) = 0$ for $\theta \geq \frac{1}{4}$, and in the simplest case of $\theta = 0$
    $$
        P(x, y \suchthat (y - x, x) \geq 0) 
        = 
        \frac{d}{2^d} \int_0^1 t^{2d - 1} dt = \frac{1}{2^{d + 1}}.
    $$

    On the other hand, for $\theta \leq -2$ we have $\sqrt{R_{\theta}(t)} \geq 1 + \frac{1}{2}t$ for all $t$, implying that the intersection~\eqref{eq:ballIntersection} is the entirity of $\mathbb{B}_d$, and hence 
    $$
        P \Big( x, y \suchthat (y - x, x) \geq \theta \Big)  
        = 
        1.
    $$
\end{proof}

The behaviour of $f_{\theta}(d)$ is illustrated in Figure~\ref{fig:fTheta} for various values of the separation threshold $\theta$.
Heuristically, we observe the following limiting behaviour:
$$
    \lim_{d \to \infty} f_{\theta}(d) =
    \begin{cases}
        1 &\text{ for } \theta < -1,
        \\
        \frac{1}{2} &\text{ for } \theta = -1,
        \\
        0 &\text{ for } \theta > -1,
    \end{cases}
$$
which may be explained by the fact that when $\theta = -1$, the surfaces of the ball $\mathbb{B}_d$ and the ball centered at $\frac{y}{2}$ meet exactly at an equator of $\mathbb{B}_d$.
The phenomenon of waist concentration (see~\cite{ledoux2001concentration}, for example) implies that in high dimensions the volume of $\mathbb{B}_d$ is concentrated around its surface and around this equator, implying that this is the threshold value of $\theta$ at which the intersection of the two balls contains slightly more than half the volume of $\mathbb{B}_d$.

What these results suggest is that for any value of $\theta \in [-1, 0]$, the function $f_{\theta}(d)$ is an invertible function of $d$, and hence could be used as the basis of a definition of intrinsic dimension.
In Definition~\ref{def:intrinsicDim} we use the behaviour at $\theta = 0$ to define our indicative notion of intrinsic dimension simply because it obviates the need to couple the scaling of the support of the distribution and the scaling of $\theta$.

\section{Few shot learning is dependent on separability}
We now consider the scenario of standard binary data classification, and show that the probability of successfully learning to classify data is intrinsically linked to the notion of relative intrinsic dimension.
We focus on the case of learning from small data sets, as in this case the link is particularly clear to demonstrate.

Mathematically, we suppose that $X$ and $Y$ are (unknown) probability distributions on an $d$-dimensional vector space $\Real^d$, and we have
a sample $\{y_{i}\}_{i=1}^{k}$ of $k$ training points sampled from $Y$ and a sample $\{x_{i}\}_{i=1}^{m}$ of $m$ training points sampled from $X$.

Since the problem setup is symmetric in the roles of $X$ and $Y$, we only analyse the influence of training data sampled from $Y$.
The role of the data sampled from $X$ (alongside any possible prior knowledge of the data distributions) is incorporated through an arbitrary but fixed point $c \in \mathbb{R}^d$ in the data space.

We consider the following linear classifier to assign the label $\ell_{X}$ to data sampled from $X$ and the label $\ell_{Y}$ to data sampled from $Y$:
\begin{equation}\label{eq:classifierDef}
    F_{\theta}(z) = 
    \begin{cases}
        \ell_{Y} &\text{if } L(z) \geq \theta,
        \\
        \ell_{X} &\text{otherwise},
    \end{cases}
\end{equation}
where $L(z) = \frac{1}{k} \sum_{i = 1}^{k} (z - y_i, y_i - c)$.
In practice, the value of the threshold $\theta$ to be used in the classifier may be determined from the training data $\{y_i\}_{i=1}^{k}$ and $\{x_i\}_{i=1}^{m}$, although here we consider it to be a free parameter of the classifier.

\begin{remark}[Comparison with similar classifiers]
    The classifier~\eqref{eq:classifierDef} may be equivalently be expressed in the form of the common Fisher discriminant with a slightly different threshold, viz.
    $$
    F_{\theta}(z) = 
        \begin{cases}
            \ell_{Y} &\text{if } (z - \mu, \mu - c) \geq \theta  + \Theta,
            \\
            \ell_{X} &\text{otherwise},
        \end{cases}
    $$\looseness=-1
    where $\mu = \frac{1}{k}\sum_{i=1}^{k} y_i$ and $\Theta = \frac{1}{k}\sum_{i=1}^{k} \| y_i \|^2- \|\mu\|^2$.
    Since the offset $\Theta$ to the threshold $\theta$ depends only on the same training data as $\theta$, it is clear that the classifier we study is simply a Fisher discriminant.
    However, we choose to write the classifier in the form~\eqref{eq:classifierDef} because it simplifies some of the forthcoming analysis.
\end{remark}

This classifier will successfully learn to classify the training data when both
$$
    P(F_{\theta}(y) = \ell_{Y}) = P(L(y) \geq \theta)
$$
is large (where the probability is taken with respect to the evaluation point $y \sim Y$ and the training data $\{y_i \sim Y\}_{i=1}^{k}$), and
$$
    P(F_{\theta}(x) = \ell_{X}) = P(L(x) < \theta)
$$
is also large (where the probability is taken with respect to the evaluation point $x \sim X$ and the training data $\{y_i \sim Y\}_{i=1}^{k}$).
We now show that both of these probabilities can be bounded from above and below by the probability of being able to separate pairs of data points by margin $\theta$.
Corollary~\ref{cor:intrinsicDimensionAndLearning} to this theorem then shows how this simply reduces to upper and lower bounds dependent on the (relative) intrinsic dimension of $Y$ and $X$ when $\theta = 0$.

\begin{theorem}[Pairwise separability and learning]\label{thm:intrinsicDimensionAndLearning}
    Let $\theta \in \Real$ and define $$p_\theta(Y, X) = P(x \sim X, y \sim Y : (x - y, y - c) \geq \theta),$$ and let $p_{\theta}(Y) = p_{\theta}(Y, Y)$.
    Then, the probability (with respect to the training sample $\{y_i \sim Y\}_{i=1}^{k}$ and the evaluation point $y \sim Y$) of successfully learning the class $Y$ is bounded by
    \begin{equation}\label{eq:generalisationTheorem:learningNew}
        p_\theta^k(Y) \leq P(F_{\theta}(y) = \ell_{Y}) \leq 1 - (1 - p_\theta(Y))^k,
    \end{equation}
    and the probability (with respect to the training sample $\{y_i \sim Y\}_{i=1}^{k}$ and the evaluation point $x \sim X$) of successfully learning the class $X$ is bounded by
    \begin{equation}\label{eq:generalisationTheorem:recallingOld}
        (1 - p_\theta(Y, X))^{k} \leq P(F_{\theta}(x) = \ell_{X}) \leq 1 - p_\theta^{k}(Y, X).
    \end{equation}
\end{theorem}
\begin{proof}
    Let $E$ be the event that $F_{\theta}(y) = \ell_{Y}$ for $y \sim Y$.
    By definition, this occurs when $y$ and $\{y_i\}_{i=1}^{k}$ are such that $\sum_{i=1}^{k} (y - y_i, y_i - c) \geq k \theta$.
    For each $1 \leq i \leq k$, let $A_i$ denote the event that $(y - y_i, y_i - c) \geq \theta$.
    Then, $\bigwedge_{i=1}^{k} A_i \Rightarrow E$ and so
    $
        P(E) \geq P(\bigwedge_{i=1}^{k} A_i).
    $
    We may further expand this using the law of total probability as
    \begin{equation}\label{eq:generalisationTheorem:introIntegral}
        P \Big( \bigwedge_{i=1}^{k} A_i \Big) 
        =
        \int_{\Real^d} P \Big( \bigwedge_{i=1}^{k} (y - y_i, y_i - c) \geq \theta \given y \Big) p(y) dy.
    \end{equation}
    Since the $\{y_i\}_{i=1}^{k}$ are independently sampled and identically distributed, it follows that the conditional probability satisfies
    \begin{align*}
        &P \Big( \{y_i \sim Y\}_{i=1}^{k} \suchthat \! \bigwedge_{i=1}^{k} (y - y_i, y_i - c) \geq \theta \given y \Big) 
        \!= \!
        P(y^{\prime} \sim Y \suchthat (y - y^{\prime}, y^{\prime} - c) \geq \theta \given y)^{k}.
    \end{align*}
    Substituting this into~\eqref{eq:generalisationTheorem:introIntegral} shows that $P \big( \bigwedge_{i=1}^{k} A_i \big) = 
        \mathbb{E}_{Y} \big[ \big( P(y^{\prime} \sim Y \suchthat (y - y^{\prime}, y^{\prime} - c) \geq \theta \given y) \big)^{k} \big]$,
    where the expectation is taken with respect to $y$.
    For a random variable $X$ and a convex function $g$, Jensen's inequality asserts that $\mathbb{E}[g(X)] \geq g(\mathbb{E}[X])$.
    Applying this here (since the function $g(x) = x^k$ is convex for $k \geq 1$), we find that
    \begin{align*}
        P \Big( \bigwedge_{i=1}^{k} A_i \Big) 
        &\geq 
        \big( \mathbb{E}_{Y} [ P(y^{\prime} \suchthat (y - y^{\prime}, y^{\prime} - c) \geq \theta \given y) ] \big)^{k}
        \\&
        = 
        \big( P(y, y^{\prime} : (y - y^{\prime}, y^{\prime} - c) \geq \theta) \big)^{k}.
    \end{align*}
    Consequently, we deduce the lower bound of~\eqref{eq:generalisationTheorem:learningNew}.
    The upper bound follows by arguing similarly and using the fact that $\bigwedge_{i=1}^{k} \operatorname{not} A_i \Rightarrow \operatorname{not} E$, from which it follows that $P(E) \leq 1 - P(\bigwedge_{i=1}^{k} \operatorname{not} A_i)$.
    An analogous argument shows the result~\eqref{eq:generalisationTheorem:recallingOld}.
\end{proof}

An immediate consequence of this theorem is that when $\theta = 0$, the probability of successfully learning can be bounded from both above and below using the (relative) intrinsic dimension of the data distributions.

\begin{corollary}[Intrinsic dimension and learning]\label{cor:intrinsicDimensionAndLearning}
    The probability (with respect to the training sample $\{y_i \sim Y\}_{i=1}^{k}$ and the evaluation point $y \sim Y$) of successfully learning the class $Y$ is bounded by
    \begin{equation}\label{eq:generalisationTheorem:idCorollary:learningNew}
        \frac{1}{2^{k(n(Y) + 1)}} \leq P(F_{0}(y) = \ell_{Y}) \leq 1 - \Big(1 - \frac{1}{2^{n(Y) + 1}}\Big)^k,
    \end{equation}
    and the probability (with respect to the training sample $\{y_i \sim Y\}_{i=1}^{k}$ and the evaluation point $x \sim X$) of successfully learning the class $X$ is bounded by
    \begin{equation*}
        1 - \Big(1 - \frac{1}{2^{n(Y, X) + 1}}\Big)^{k} \leq P(F_{0}(x) = \ell_{X}) \leq \frac{1}{2^{k(n(Y, X) + 1)}}
    \end{equation*}
\end{corollary}

We note that the best lower bound which can be shown by~\eqref{eq:generalisationTheorem:idCorollary:learningNew} is $\frac{1}{2}$, due to the fact thatthe classifier with $\theta = 0$ will pass through the centre of the $Y$ distribution. 
Despite this, Corollary~\ref{cor:intrinsicDimensionAndLearning} shows that the intrinsic dimension of $Y$ is sufficient to know whether the probability of correctly learning the class $Y$ is less than $\frac{1}{2}$.
Arguing symmetricaly, a more refined analysis taking more account of the training set $\{x_i\}_{i=1}^{m}$ could instead show a version of the bound~\eqref{eq:generalisationTheorem:idCorollary:learningNew} which depends on the relative intrinsic dimension $n(X, Y)$.

These bounds are tuned to the case when the size $k$ of the training set sampled from $Y$ is small, and the upper and lower bounds separate from each other as $k$ grows, and alternative arguments would be required to get sharp bounds in the case of large $k$.
However, even for large values of $k$, if the (relative) intrinsic dimension of the data distributions is sufficiently large or small, the bounds above will provide tight guarantees on the success of learning.

\section{Learning with polynomial kernels}\label{sec:polynomialKernels}

As an application of our proposed notion of intrinsic dimension, we use it to find the optimal polynomial kernel for a classification problem --- i.e. the degree of the polynomial feature map in which two data sets become easiest to separate.

For fixed bias $b > 1$ and polynomial degree $k \geq 0$, let the polynomial kernel $\kappa : \Real^d \times \Real^d \to \Real$ be given by
\begin{equation}\label{eq:polyKernel}
    \kappa(x, y) = (b^2 + x \cdot y)^k.
\end{equation}
There exists a polynomial feature map $\phi : \Real^d \to \Real^N$, where $N = \binom{d + k}{k}$, such that $\kappa(x, y) = (\phi(x), \phi(y))$ (see~\cite{fewShotNonlinear}, for example, for details).

Consider
$$
    P(x, y, \sim \mathcal{U}(\mathbb{B}_d) : (\phi(x) - \phi(y), \phi(y) - c) \geq \theta),
$$
where $c = \frac{1}{\ballvolume(1)} \int_{\mathbb{B}_d} \phi(z) dz$ is the empirical mean of the data in feature space.
Then, expanding the inner product,
\begin{align*}
    &(\phi(x) - \phi(y), \phi(y) - c) 
    = 
    k(x, y) - k(y, y)
    +
    \int_{\mathbb{B}_d} \frac{k(y, z) - k(x, z)}{\ballvolume(1)} dz
    \\\qquad&
    = 
    (b^2 + x \cdot y)^k 
    - 
    (b^2 + \|y\|^2)^k
    +
    \int_{\mathbb{B}_d} \frac{
        (b^2 + y \cdot z)^k -
        (b^2 + x \cdot z)^k 
        }{\ballvolume(1)} 
    dz.
\end{align*}
Exploiting the spherical symmetry of $\mathcal{U}(\mathbb{B}_d)$, we have
\begin{align*}
    &\frac{1}{\ballvolume(1)}
    \int_{\mathbb{B}_d} 
        (b^2 + x \cdot z)^k
    dz
    =
    \int_{-1}^{1}
    \frac{\ballvolume[d-1]((1 - t^2)^{1/2})}{\ballvolume(1)} 
    (b^2 + t\|x\|)^k
    dt
    =
    q(\|x\|),
\end{align*}
for $b \geq 1$, where $q : [0, 1] \to \Real$ is given by
    $
    q(\|x\|)
    :=
    b^{2k}
    \hypgeometric \Big( \frac{1 - k}{2}, -\frac{k}{2}; \frac{d}{2} + 1; \frac{\|x\|^2}{b^4} \Big),
    $
with $\hypgeometric$ denoting the hypergeometric function.
Therefore
$(\phi(x) - \phi(y), \phi(y) - c) \geq \theta$ if and only if
\begin{align*}
    \cos(\beta(x, y)) 
    &\geq
    Q(\|x\|, \|y\|)
\end{align*}
where $\beta(x, y) = \arccos(\frac{(x, y)}{\|x\|\|y\|})$ denotes the angle between $x$ and $y$, and
\begin{align*}
    Q(s, t)
    :=
    (st)^{-1}
    \Big(
        \big(
        \theta
        +
        (b^2 + t^2)^k
        + q(s) - q(t)
        \big)^{1/k}
        -
        b^2
    \Big).
\end{align*}
Geometric arguments show that for any $\alpha \in [-1, 1]$,
$$
    P(x, y \sim \mathcal{U}(\mathbb{B}_d) \suchthat \cos(\beta(x, y)) \geq \alpha \given \|x\|, \|y\|)
    =
    \capareafraction(\alpha)
$$
where $\capareafraction(\alpha)$ denotes the proportion of the surface area of a unit sphere which falls within a spherical cap with opening angle $\arccos(\alpha)$, given for $d > 1$ by
$$
    \capareafraction(\alpha)
    =
    \begin{cases}
        0,
        &\alpha > 1,
        \\
        \frac{1}{2}I_{(\sin(\arccos(\alpha)))^2}\Big( \frac{d - 1}{2}, \frac{1}{2} \Big),
        &\alpha \in [0, 1],
        \\
        1 - \capareafraction(-\alpha),
        &\alpha \in (-1, 0),
        \\
        1,
        &\alpha \leq -1,
    \end{cases}
$$
where $I_x(a, b)$ is the regulalised incomplete beta function, and for $d = 1$ by
$$
    \capareafraction[1](\alpha)
    =
    \begin{cases}
        0
        \text{ for } \alpha > 1; \quad
        \frac{1}{2}
        \text{ for } \alpha \in (-1, 1]; \quad
        1
        \text{ for } \alpha \leq -1
    \end{cases}
$$
Let $E$ be the event that $x, y \sim \mathcal{U}(\mathbb{B}_d)$ are such that $\cos(\beta) \geq Q(\|x\|, \|y\|)$. 
Then, by the law of total probability,
$$
    P(E)
    =
    \int_{0}^1
    \int_{0}^1
    P(E \given \|x\| = s, \|y\| = t)
    \hat{p}(s) \hat{p}(t)
    ds dt,
$$
where $\hat{p}(t) = \frac{\ballarea(t)}{\ballvolume(1)} = dt^{d-1}$ denotes the density associated with $\|z\|$ for $z \sim \mathcal{U}(\mathbb{B}_d)$.

The arguments above therefore prove the following theorem, from which Theorem~\ref{thm:separability} arises as a simplified special case when $k = 1$

\begin{theorem}[Separability in polynomial feature space]\label{thm:polySeparability}
    Let $k > 0$, let $d$ be a fixed positive integer, and let $\phi$ denote the feature map associated with the polynomial kernel~\eqref{eq:polyKernel} with degree $k$ in dimension $d$.
    Then, for $\theta \in \Real$,
    \begin{align*}
        &P(x, y \sim \mathcal{U}(\mathbb{B}_d) : (\phi(x) - \phi(y), \phi(y) - c) \geq \theta) 
       \\&\qquad
        =
        d^2
        \int_{0}^1
        \int_{0}^1
        \capareafraction(Q(s, t))
        s^{d-1}
        t^{d-1}
        ds dt.
    \end{align*}
\end{theorem}

\begin{figure}
	\centering
    \includegraphics[width=0.6\linewidth]{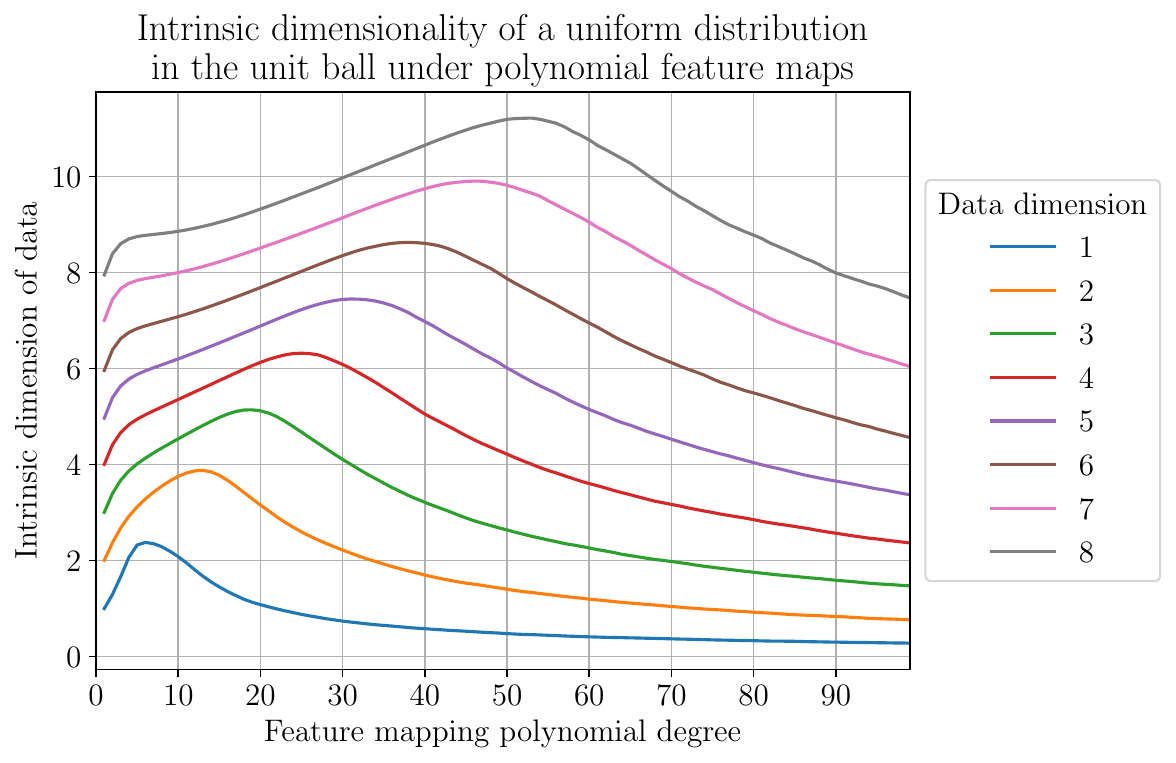}
    \caption{The intrinsic dimension of the image of $\mathcal{U}(\mathbb{B}_d)$ under a polynomial feature map, for different polynomial degrees and data space dimensions $d$.}
    \label{fig:polynomialIntrinsicDimension}
\end{figure}

Figure~\ref{fig:polynomialIntrinsicDimension} shows how the intrinsic dimension of the unit ball in various dimensions is affected by applying a polynomial feature mapping.
Since the degree $k$ polynomial feature map $\phi : \mathbb{R}^d \to \mathbb{R}^N$, where $N = \binom{d + k}{k}$, increases the apparent dimension of the space as $k$ increases, the rule of thumb encapsulated by the blessing of dimensionality would lead us to expect that high order polynomial kernels should make the data more separable.
However, this is not what we observe.
Instead, the intrinsic dimension reveals that there is an `optimal' polynomial degree, for which the data is most separable, and increasing the polynomial degree further beyond the point can actually have the detrimental effect of making the data less separable.

\section{Conclusion}
We have introduced a new notion of the intrinsic dimension of a data distribution, based on the pairwise separability properties of data points sampled from this distribution.
Alongside this, we have also introduced a notion of the relative intrinsic dimension of a data distribution relative to another distribution.
Theorem~\ref{thm:intrinsicDimensionAndLearning} shows how these notions of intrinsic dimension occupy a fundamental position in the theory of learning, as they directly provide upper and lower bounds on the probability of successfully learning in a generalisable fashion.

Many open questions remain, however, such as how to accurately determine the intrinsic dimension of a data distribution using just sampled data from that distribution, and how best to utilise these insights to improve neural network learning.
This work also opens to door to generalising the concept beyond just simple linear functionals of the data distribution to notions of intrinsic dimensionality based around other more interesting models.
The idea also generalises beyond examining individual points sampled from distributions, to studying the collective behaviour of groups, or `granules' of sampled data.

\subsubsection{Acknowledgements} The authors are grateful for financial support by  the UKRI and EPSRC (UKRI Turing AI Fellowship ARaISE EP/V025295/1). I.Y.T. is also grateful for support from the UKRI Trustworthy Autonomous Systems Node in Verifiability EP/V026801/1.

\bibliographystyle{splncs04}
\bibliography{references.bib}

\end{document}